\makeatletter
\@ifundefined{@parse@version@dash}{%
	\def\@parse@version#1{\@parse@version@0#1}
	\def\@parse@version@#1/#2/#3#4#5\@nil{%
		\@parse@version@dash#1-#2-#3#4\@nil}
	\def\@parse@version@dash#1-#2-#3#4#5\@nil{%
		\if\relax#2\relax\else#1\fi#2#3#4 }
}{}
\makeatother

\documentclass[%
reprint,
superscriptaddress,
amsmath,amssymb,
prl,
showkeys
]{revtex4-2}

\usepackage{color}
\usepackage{graphicx}
\usepackage{dcolumn}
\usepackage{bm}
\usepackage{hyperref}
\usepackage{amsmath,amssymb,amsthm,mathrsfs,bm,graphicx,xcolor,hyperref,microtype}
\usepackage{epstopdf}
\newcommand{\dd}{\,\mathrm d}

\newtheorem{theorem}{Theorem}

\newcommand{\RR}{\mathbb R}

\newcommand{\thv}{\bm{\theta}}
\newcommand{\uv}{\bm{u}}
\newcommand{\zv}{\bm{z}}
\newcommand{\Fv}{\bm{F}}
\newcommand{\Gv}{\bm{G}}

\begin{document}

\title{Some coarse-grain laws for macroscopic states of neural networks}

\title{Macroscopic states and dynamical coarse-graining of neural-network training}
\title{Dynamical coarse-graining of neural-network training}

\author{Lin-Juan Wang}
\email{wanglj@buaa.edu.cn}
\affiliation{School of Astronautics, Beihang University, Beijing 100191, China}
\author{Da Li}
\affiliation{School of Astronautics, Beihang University, Beijing 100191, China}
\author{Huiling Duan}
\author{Jianxiang Wang}
\affiliation{School of Mechanics and Engineering Science, Peking University, Beijing 100871, China}
\affiliation{CAPT-HEDPS, and IFSA Collaborative Innovation Center of MoE, Peking University, Beijing 100871, China}

\date{\today}

\begin{abstract}
Dynamical coarse-graining of neural networks gives macroscopic observables and enables training to be predicted, and controlled
in a state space far smaller than the parameter space.
In this paper, we first define the learning operator following the training law of neural networks and establish a descent theorem for the learning operator, which serves to determine whether the operator is sufficient to characterize the learning dynamics.
Then the exact closure theorem and the dynamic coarse-graining theorem are presented in sequence. The three theorems established in this paper provide a systematic framework for constructing and assessing dynamical coarse-grained models of neural-network training. Finally, we demonstrate the applicability of the theorems with a finite-width nonlinear ReLU network. Its dynamic coarse-grained model for a fixed pattern is exactly derived and composed of five independent moments. The model exactly predict the training outputs and evolving learning operator. The framework may also guide the development of approximate coarse-grained models for more general neural networks when exact closure is unavailable.

\end{abstract}


\maketitle


While scaling laws for artificial neural networks characterize how the model performance evolves after sufficiently long training, a central open question arises for neural networks with billions or even trillions of parameters: do there exist coarse-grained evolution laws that describe their macroscopic observables/states on a far lower-dimensional state space than the original parameter space, and predict the subsequent training dynamics? If such laws exist, are there general conditions that they must satisfy? If this is achievable, training processes could be predicted, compared, and controlled within a state space greatly smaller than the full parameter space.

There have been some theories of the dynamical reduction of neural networks in the literature. Stochastic descriptions approximate minibatch stochastic gradient descent (SGD) by diffusion processes under specified scaling assumptions, allowing the Fokker--Planck analyses of nonequilibrium fluctuations and implicit regularization~\cite{Mandt2017,Li2017,Chaudhari2018,Yang2023}. The statistical mechanics approach also supplies deterministic equations and predictive effective descriptions of training and transfer~\cite{SaadSolla1995,Goldt2020,Baglioni2024,Ingrosso2025,Aiudi2025}. Large-width limits yield measure-valued mean-field dynamics, fixed neural tangent kernel (NTK) dynamics in the appropriate regime, and feature-learning limits~\cite{Mei2018,Chizat2018,Jacot2018,Chizat2019,YangHu2021}. The dynamical mean-field theory further resolves evolving kernels and finite-width fluctuations~\cite{Bordelon2023,Bordelon2023finite}. Other reductions exploit intrinsic parametrizations, invariant structures, or fast--slow separation~\cite{Fenichel1979,Saxe2014,Lin2021}. Notably, recent research constructs architecture-dependent intrinsic gradient flows for deep networks~\cite{Marcotte2026}. Physical analogies can also provide predictive effective training models~\cite{Shi2025}.
However, these dynamical reductions are not necessarily dynamical coarse-grained approaches for neural-network training.
A basic question for dynamical coarse-graining is whether a prescribed finite-width training observation admits a trajectory-wise exact autonomous state, and whether its dimension is genuinely smaller than the microscopic model.


In this letter, we address the mathematical requirements for a dynamical coarse-graining of neural-network training. These requirements include:(i) \emph{state sufficiency}--- whether the selected sate is sufficient to characterize the learning dynamics? (ii) \emph{exact closure}---
can a closed macroscopic dynamics for the sufficient state be constructed? (iii) \emph{dynamical coarse-graining}---does the resulting closed dynamics achieve genuine compression and scale consistency? The three issues will be addressed generally, and then the developed coarse-graining equations will be tested within a finite-width nonlinear ReLU neural-network.  

\begin{figure*}[t]
\centering\includegraphics[width=1\textwidth]{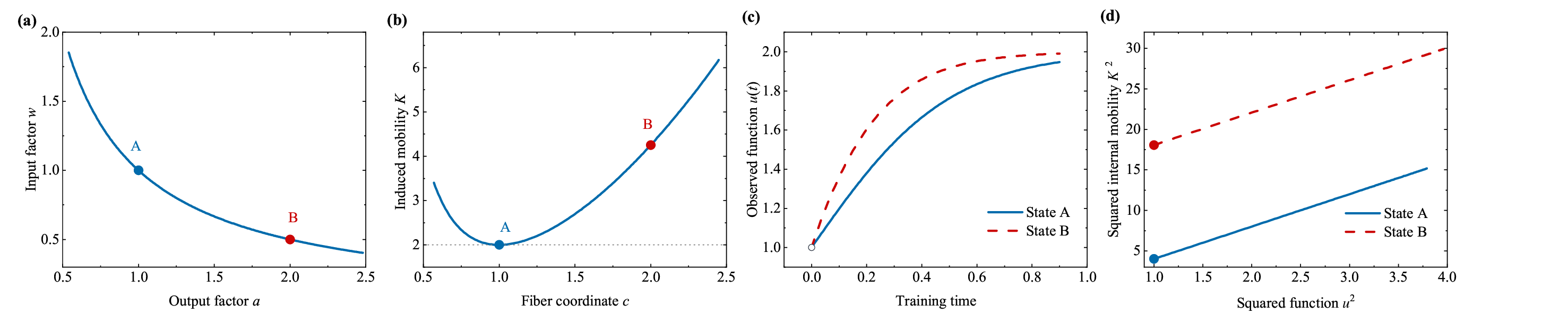}
\caption{\label{fig:fiber}(a) The fiber $aw=1$ in parameter space. (b) Learning operator $K$ varies along this fiber. (c) Output functions $u(t)$ with identical initial function values evolve differently along the training time. (d) The augmented state $(u,K)$ has an autonomous flow, and the trajectories of parameter states A and B are parallel in the $(u^2,K^2)$ plane.}  
\end{figure*}

\emph{Neural-network training and the observable equation.---}
We consider a parameterized neural network $f_{\thv}:\mathcal X\to\mathbb R$
with scalar output and prescribed input domain $\mathcal X$ (vector
outputs can be studied in the same way). We refer to its trainable parameters as the microscopic state $\thv=(\theta_1,\ldots,\theta_P)^{\mathsf T}
\in\Theta$, where $P$ is the total number of trainable scalars and
$\Theta\subset\RR^P$ is an open set of admissible parameter values.
$\thv$ denotes one particular parameter vector, and
$\Theta$ denotes the domain of $\thv$. Fix $N$ prescribed inputs
$x_1,\ldots,x_N\in\mathcal X$. Their outputs form the observation vector
$\uv=h(\thv)=(f_{\thv}(x_1),\ldots,f_{\thv}(x_N))^{\mathsf T}$.
Here $h:\Theta\to\mathcal U$ is the observation map and
$\mathcal U\subset\RR^N$ is an open domain containing the image of $\Theta$.
We assume that $h$ is continuously differentiable on the considered domain.
Its Jacobian $\mathbf J_{\thv}:=Dh(\thv)\in\RR^{N\times P}$ has entries
$(\mathbf J_{\thv})_{ji}=\partial h_j/\partial\theta_i$.
The relation $\delta\uv=\mathbf J_{\thv}\delta\thv$ is the
linearized output response to a parameter perturbation.
Let $\ell:\mathcal U\to\RR$ be the differentiable loss, and assume that
the parameter objective depends on $\thv$ through the observation map, i.e., $\mathcal L(\thv):=\ell(h(\thv))$. 
The ordinary Euclidean chain rule yields
\begin{equation}
 \nabla_{\thv}\mathcal L(\thv)
 =\mathbf J_{\thv}^{\mathsf T}\nabla_{\uv}\ell(\uv).
 \label{eq:chainrule}
\end{equation}
Choose a prescribed smooth, symmetric positive-semidefinite mobility
matrix $\mathbf M_{\thv}\in\RR^{P\times P}$, independent of the particular
loss within any task family considered below. For the continuous,
full-batch gradient-type training law, there is
$\dot\thv=-\mathbf M_{\thv}\nabla_{\thv}\mathcal L$. Substitution of
Eq.~(\ref{eq:chainrule}) into
$\dot\uv=\mathbf J_{\thv}\dot\thv$ gives the exact projected equation
\begin{equation}
 \dot\uv=-\mathbf J_{\thv}\mathbf M_{\thv}
 \mathbf J_{\thv}^{\mathsf T}\nabla_{\uv}\ell(\uv).
 \label{eq:observed}
\end{equation}

We see that the symmetric positive semidefinite matrix
$\mathbf K_{\thv}:=\mathbf J_{\thv}\mathbf M_{\thv}
 \mathbf J_{\thv}^{\mathsf T}\in\RR^{N\times N}$ plays the key role that relates the velocity $\dot\uv$ to the function $\ell(\uv)$. Thus, we define it as  the \emph{learning operator}.
In the Euclidean gradient flow, $\mathbf M_{\thv}=\mathbf I_P$, and
$\mathbf K_{\thv}=\mathbf J_{\thv}\mathbf J_{\thv}^{\mathsf T}$
is the empirical NTK~\cite{Jacot2018} . Its $(j,k)$
entry measures the inner product of the parameter sensitivities of the
outputs at $x_j$ and $x_k$. Eq.~(\ref{eq:observed}) is rewritten as
$\dot\uv=-\mathbf K_{\thv}\nabla_{\uv}\ell(\uv)$ which is not
generally a closed set of equations with respect to $\uv$, because $\mathbf K_{\thv}$ depends on which microscopic parameter realization produces the same observed output $\uv$.
For $\uv_0\in\mathcal U$, define the fiber of the observation map as
\begin{equation}
 h^{-1}(\uv_0):=\{\thv\in\Theta:h(\thv)=\uv_0\}.
 \label{eq:fiber}
\end{equation}
It collects all microscopic parameter vectors producing the same
specified $N$ outputs. The subspace $\ker\mathbf J_{\thv}$ comprises
parameter perturbations invisible to these outputs.
The closure of Eq.~(\ref{eq:observed}) depends on the property of the learning operator $\mathbf K_{\thv}$. Thus we propose the following theorem related to operator descent of $\mathbf K_{\thv}$ as follows.

\begin{theorem}[Learning operator descent]\label{thm:descent}
Let $h:\Theta\to\mathcal U$ be a smooth surjective submersion and
$\mathbf K_{\thv}:\Theta\to\operatorname{Sym}_N$ smooth, where
$\operatorname{Sym}_N$ is the space of real symmetric $N\times N$ matrices.
The following are equivalent:
\emph{(i)} there is a unique smooth
$\overline{\mathbf K}:\mathcal U\to\operatorname{Sym}_N$ with
$\mathbf K_{\thv}=\overline{\mathbf K}(h(\thv))$;
\emph{(ii)} $\mathbf K_{\thv_1}=\mathbf K_{\thv_2}$ whenever
$h(\thv_1)=h(\thv_2)$.
If every fiber is connected, these are further equivalent to
$D\mathbf K_{\thv}[\mathbf v]=\mathbf0$ for all
$\mathbf v\in\ker\mathbf J_{\thv}$.
For a family $\mathscr F$ of smooth losses whose gradients
$\{\nabla_{\uv}\ell(\uv):\ell\in\mathscr F\}$ span $\RR^N$
at every $\uv$, (i)--(ii) are equivalent to fiberwise constancy of
$-\mathbf K_{\thv}\nabla_{\uv}\ell(h(\thv))$ for every
$\ell\in\mathscr F$.
\end{theorem}

Condition (ii) states that two networks which are indistinguishable by the retained
outputs also have identical instantaneous learning operators. 
For one fixed loss, it is sufficient to require constancy of the
velocity $\mathbf K_{\thv}\nabla_{\uv}\ell(\uv)$, not of
all entries of $\mathbf K_{\thv}$. The proof of Theorem~\ref{thm:descent} is given in the Appendix.

As a special case, we take the two-weight scalar network $q=aw$ on the regular positive domain $a,w>0$, with $P=2$, $N=1$, $\thv=(a,w)^{\mathsf T}$, $\uv=h(\thv)=q$, and $\mathbf M_{\thv}=\mathbf I_2$. Thus $\mathbf J_{\thv}=(w,a)$ and $\mathbf K_{\thv}=w^2+a^2$.
In this special case, the observation vector $\uv$ and learning operator $\mathbf K_{\thv}$ are scalars, respectively denoted as $u$ and $K$.
This is the simplest factorized two-weight network, where $a$ and $w$ are scalar weights and the desired output is $2$. For $\ell(u)=\tfrac12(u-2)^2$, the chain rule gives $\dot a=-(u-2)w$ and $\dot w=-(u-2)a$. Therefore $\dot u=-(u-2)(a^2+w^2)$. Two parameter states, A: $(a,w)=(1,1)$ and B: $(2,1/2)$, have the same initial output $u=1$ but different initial learning operator, $K_A=2$ and $K_B=4.25$; consequently $\dot u_A(0)=2$ and $\dot u_B(0)=4.25$. Fig.1(a) and (b) show the same-output curve/fiber $aw=1$ and the changing learning operator $K$ along the fiber.  Its evolution, obtained by differentiating $a^2+w^2$, satisfies
\begin{equation}
\dot u=-K(u-2),\qquad\dot K=-4u(u-2).
 \label{eq:scalarcomplete}
\end{equation}
Thus the sate $(u,K)$ retains sufficient information to predict subsequent training dynamics.
This gives closure but does not realize dimensionality reduction. Both the parameters $(a,w)$ and the completed state $(u,K)$ have dimension two.
Fig.1(c) integrates the two resulting trajectories of both parameter sates A and B.
Eq.(\ref{eq:scalarcomplete}) indicates
$\frac{\rm{d}}{{\rm{d}}t}(K^2-4u^2)=0$ which means $K^2-4u^2=\rm{constant}$ for any parameter state.
Therefore, the trajectories of both parameter sates A and B are parallel in the $(u^2,K^2)$ plane as shown in Fig.1(d).
In general cases, the sate $(u,K)$ is not necessarily sufficient, and additional states is required. The theorem related to the exact closure of sates is proposed as follows.

\begin{figure*}[t]
\centering\includegraphics[width=.96\textwidth]{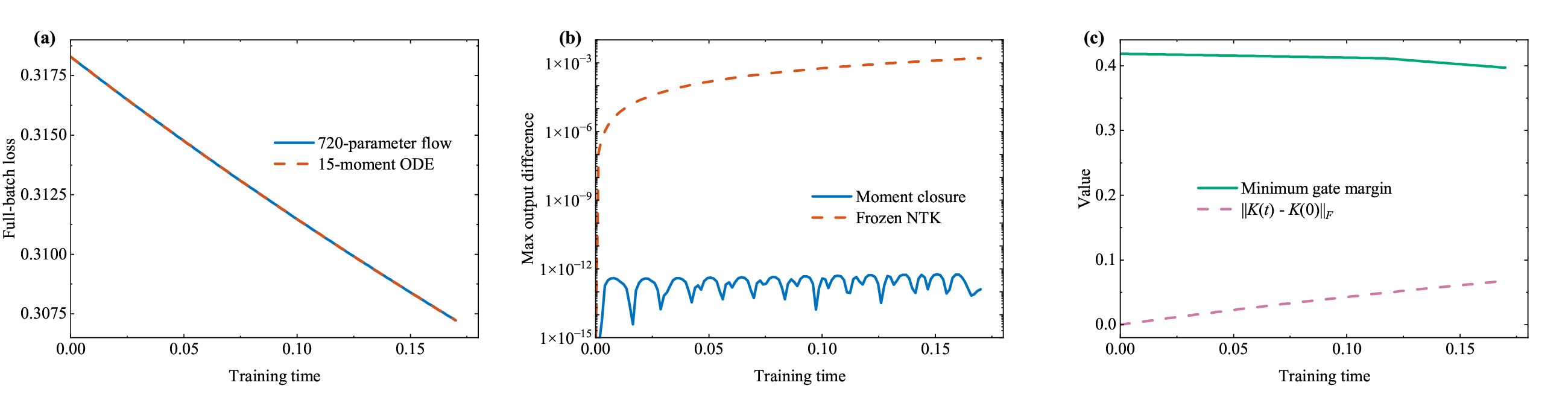}
\caption{\label{fig:relu}Exact closure of a two-layer network training with 240 hidden ReLU neurons. (a) Loss from a 720-parameter ReLU gradient flow and from its 15-moment macroscopic equation. (b) The moment macroscopic equation stays near numerical precision whereas a frozen-NTK approximation deviates. (c) The training kernel changes although the minimum distance to every gate boundary remains positive.}
\end{figure*}

\emph{Exact closure.---}
The microscopic training equation in the general form is written as
$\dot\thv=\Fv(\thv)$, where
$\Fv:\Theta\to\RR^P$ is a smooth vector-valued function. 
Let the scalar observable $\phi_i:\Theta\to\RR$, $i=1,\ldots,r$, be smooth scalar functions and
collect them into the candidate macrostate
$\zv=\Pi(\thv):=(\phi_1(\thv),\ldots,\phi_r(\thv))^{\mathsf T}$.
Here $Z\subset\RR^r$ is an open domain of admissible macrostates,
$\Pi:\Theta\to Z$ is the state projection, and
$\mathcal R:Z\to\mathcal U$ is an output readout, assumed to obey
$h=\mathcal R\circ\Pi$ whenever the original outputs must be reconstructed.
The derivative of $\phi_i$ along a microscopic trajectory is the Lie derivative $\mathcal D_{\Fv}\phi_i(\thv)
 :=D\phi_i(\thv)\Fv(\thv)$.

\begin{theorem}[Exact closure]\label{thm:completion}
Suppose $\Pi:\Theta\to Z$ is a smooth surjective submersion. The following are equivalent:
\emph{(i)} a unique smooth $\Gv:Z\to\RR^r$ exists with $ D\Pi(\thv)\Fv(\thv)=\Gv(\Pi(\thv))$;
\emph{(ii)} the vector $D\Pi(\thv)\Fv(\thv)$ is constant on
every fiber $\Pi^{-1}(\zv)$;
\emph{(iii)} for each $i=1,\ldots,r$, a smooth scalar function $G_i:Z\to\RR$ exists with $\mathcal D_{\Fv}\phi_i(\thv)=G_i(\Pi(\thv))$.
\end{theorem}

The proof of Theorem~\ref{thm:completion} is given in the Appendix. $\Gv=(G_1,\ldots,G_r)^{\mathsf T}$ is an ordinary vector in $\RR^r$, not a new state space.
Take a two-layer network with $m$ hidden ReLU neurons and scalar input $x$. In this example $f_m(x)$ abbreviates the general network notation $f_{\thv}(x)$, with
\begin{equation}
 f_m(x)=\frac1m\sum_{i=1}^{m}a_i(w_ix+b_i)_+,
 \quad (r)_+:=\max\{r,0\},
 \label{eq:reludef}
\end{equation}
where $a_i$ is the output weight, $w_i$ the input weight, and $b_i$ the bias of neuron $i$; the microscopic state therefore has $3m$ parameters. On $N$ prescribed input--target pairs $(x_j,y_j)$, let
\begin{equation}
 e_j:=f_m(x_j)-y_j,\quad
 \ell=\tfrac12\sum_{j=1}^{N}e_j^2,\quad
 \dot\thv=-m\nabla_{\thv}\ell.
 \label{eq:reluobjective}
\end{equation}
This convention means $\mathbf M_{\thv}=m\mathbf I_{3m}$ and thus $\mathbf K_{\thv}=m\mathbf J_{\thv}\mathbf J_{\thv}^{\mathsf T}$. The preactivation $r_{ij}:=w_ix_j+b_i$ determines whether neuron $i$ contributes at sample $j$; define the activation gate $s_{ij}:=\mathbf1_{\{r_{ij}>0\}}\in\{0,1\}$. Within a time interval where no $r_{ij}$ crosses zero, neurons sharing the same gate vector $(s_{i1},\ldots,s_{iN})$ form a pattern $\alpha$, with index set $I_\alpha$ and common indicators $s_{\alpha j}$. 
Define
\begin{equation}
 p_\alpha:=-\sum_j e_js_{\alpha j}x_j,\qquad
 q_\alpha:=-\sum_j e_js_{\alpha j}.
 \label{eq:pqforces}
\end{equation}
Within each group, $\dot a_i=p_\alpha w_i+q_\alpha b_i$, $\dot w_i=a_i p_\alpha$, and $\dot b_i=a_i q_\alpha$ for every $i\in I_\alpha$.
Using $\langle g\rangle_\alpha:=m^{-1}\sum_{i\in I_\alpha}g_i$, retain five pattern moments
\begin{equation}
\begin{aligned}
 P_\alpha&=\langle aw\rangle_\alpha,& Q_\alpha&=\langle ab\rangle_\alpha,\\
 U_\alpha&=\langle a^2+w^2\rangle_\alpha,&
 V_\alpha&=\langle a^2+b^2\rangle_\alpha,\\
 C_\alpha&=\langle wb\rangle_\alpha.&&
\end{aligned}
\label{eq:five}
\end{equation}
Differentiation gives
\begin{equation}\label{EQODE}
\begin{aligned}
\dot P_\alpha&=U_\alpha p_\alpha+C_\alpha q_\alpha,&
\dot Q_\alpha&=C_\alpha p_\alpha+V_\alpha q_\alpha,\\
\dot U_\alpha&=4P_\alpha p_\alpha+2Q_\alpha q_\alpha,&
\dot V_\alpha&=2P_\alpha p_\alpha+4Q_\alpha q_\alpha,\\
\dot C_\alpha&=P_\alpha q_\alpha+Q_\alpha p_\alpha,&&
\end{aligned}
\end{equation}
The output and training operator are functions of these moments:
\begin{equation}
\begin{split}
 f_m(x_j)&=\sum_\alpha s_{\alpha j}(P_\alpha x_j+Q_\alpha),\\
 (\mathbf K)_{jk}&=\sum_\alpha s_{\alpha j}s_{\alpha k}
 \{U_\alpha x_jx_k+C_\alpha(x_j+x_k)+V_\alpha\}.
\end{split}
\end{equation}
Eq.~(\ref{EQODE}) is a closed nonlinear ODE with five scalars $(P_{\alpha},Q_{\alpha},U_{\alpha},V_{\alpha},C_{\alpha})$, which is valid for every fixed gate pattern. The same five moments work for fixed slope piecewise homogeneous-linear activations (such as leaky ReLU); general piecewise-affine activations admit an eight-moment completion per pattern. No generic five moment closure is asserted for smooth activations such
as GELU or SiLU.

For Fig.~\ref{fig:relu}, $m=240$, $N=3$, $\bm x=(-1,0,1)^{\mathsf T}$, and $\bm y=(0.65,-0.35,0.55)^{\mathsf T}$,
with three occupied patterns. The full 720-dimensional parameter flow and its 15-dimensional macroscopic sate ODE are independently integrated on $0\le t\le0.17$ from matching initial moments.  Fig.~\ref{fig:relu}(a) compares the training losses predicted by the 15-moment macroscopic equation with those obtained from the direct 720-dimensional ReLU gradient flow, showing agreement within numerical accuracy.
Fig.~\ref{fig:relu}(b) plots the maximum output error from our moment closure and the frozen NTK $\dot\uv=-\mathbf K(0)(\uv-\bm y)$.
The moment closure achieves an error of approximately $10^{-13}$, whereas the error of the frozen-NTK  increases over the course of training.
The green line in Fig.~\ref{fig:relu}(c) verifies that the minimum gate margin $\min_{ij}|r_{ij}|$ stays positive, which makes sure the moment closure is within a valid range and no gate event occurs.
As shown by the red curve in Fig.~\ref{fig:relu}(c), the instantaneous learning operator deviates from its initial value $K(0)$ during training, which explains the discrepancy between the frozen-NTK approximation and the direct parameter training dynamics.

\emph{Dynamical coarse-graining.---}
Exact closure is only the first requirement for coarse-graining, and the dynamics may still involve as many state
variables as the original parameter model. We need to further discuss the issue of coarse-grained dimensionality reduction.
Let $P_m=\dim\Theta_m$ denote the microscopic parameter count for a width-$m$ family, distinct from the scalar moment $P_\alpha$. A reduced state must remain sufficiently small as $P_m\to\infty$ and its evolution should be stable under changes of width. These are additional conditions beyond exact closure.

\begin{theorem}[Dynamical coarse-graining and observational rank]\label{thm:scale}
\emph{(a) Scale stability.} Let $P_m\to\infty$ and fix $r\ge1$.
Suppose smooth state projections
$\Pi_m:\Theta_m\subset\RR^{P_m}\to Z\subset\RR^r$
yield exact reduced ODEs
$\dot\zv_m=\Gv_m(\zv_m)$, where
$\Gv_m,\Gv:Z\to\RR^r$. Let
$\zv_m:[0,T]\to Z$ and $\zv:[0,T]\to Z$ be respective
solutions. Suppose a common set $\mathcal C\subset Z$ contains all
these trajectories and, for $\mathbf a,\mathbf b\in\mathcal C$,
$\|\Gv_m(\mathbf a)-\Gv_m(\mathbf b)\|
\le L\|\mathbf a-\mathbf b\|$, with $L$ independent of $m$.
Set $\delta_m:=\|\zv_m(0)-\zv(0)\|$ and
$\epsilon_m:=\sup_{\mathbf a\in\mathcal C}
\|\Gv_m(\mathbf a)-\Gv(\mathbf a)\|$.
If $\delta_m,\epsilon_m\to0$, then
\begin{equation}
\sup_{0\le t\le T}\|\zv_m(t)-\zv(t)\|
\le e^{LT}(\delta_m+T\epsilon_m)\to0,
\quad \frac r{P_m}\to0.
\label{eq:coarse}
\end{equation}
\emph{(b) Observation-rank obstruction.} Let
$\mathcal O_m:\Theta_m\to\RR^{d_m}$ be smooth. If
$\mathcal O_m=\mathcal R_m\circ\widehat\Pi_m$
on an open domain, with smooth maps
$\widehat\Pi_m:\Theta_m\to\widehat Z_m\subset\RR^{r_m}$
and $\mathcal R_m:\widehat Z_m\to\RR^{d_m}$, then
\begin{equation}
\operatorname{rank}D\mathcal O_m(\thv)\le r_m
\quad(\thv\in\Theta_m).
\label{eq:rank}
\end{equation}
\end{theorem}

In part (a), $\zv_m(t)$ is the reduced trajectory of a width-$m$ model, $\zv(t)$ its proposed limit, $\Gv_m$ and $\Gv$ their reduced vector fields, $L$ a common Lipschitz constant, $T$ a fixed prediction horizon, $\delta_m$ the initial-state discrepancy and $\epsilon_m$ the largest vector-field discrepancy over the region visited. The bound follows by comparing the two integral ODEs and applying Gr\"onwall's inequality; it quantifies how errors accumulate rather than merely asserting convergence. In part (b), $\mathcal O_m$ is the requested observation map, $\widehat\Pi_m$ is the candidate macrostate, and $\mathcal R_m$ reconstructs the observation from it. Differentiation gives $D\mathcal O_m=D\mathcal R_m\,D\widehat\Pi_m$; hence the number $r_m$ of state coordinates cannot be smaller than the number of independent observation sensitivities, $\operatorname{rank}D\mathcal O_m$.

For the fixed-pattern ReLU model, the reduced vector field is the same for every width, so $\epsilon_m=0$ in part (a). Under independent, uniformly bounded neuron initializations within each pattern, each retained moment is an average of independent contributions. Consequently its variance scales as $1/m$, giving a typical $m^{-1/2}$ initial error; the ODE stability estimate propagates this error over a finite time interval. Fig.~\ref{fig:width} uses widths $m=60,120,240,480,960,1920$ and 36 seeded initializations per width, with the same three training inputs and a time horizon $T=0.17$. Fig.~\ref{fig:width}(a) overlays representative outputs with the deterministic population-moment solution. Fig.~\ref{fig:width}(b) and (c) plot the ensemble-mean initial moment discrepancy and the largest-in-time output discrepancy. The fitted slopes, about $-0.49$ and $-0.50$, are consistent with the $-1/2$ prediction.

\begin{figure*}[t]
\centering\includegraphics[width=.96\textwidth]{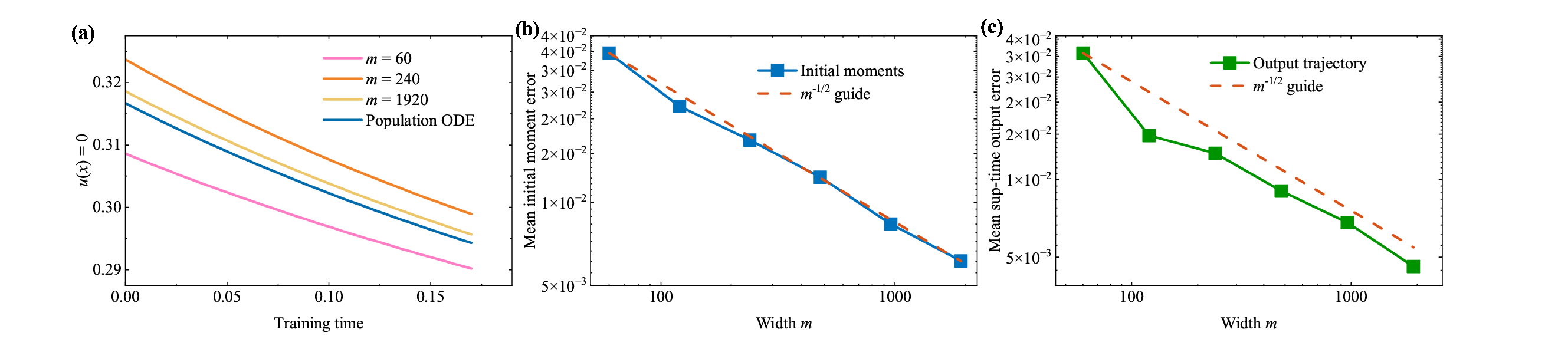}
\caption{\label{fig:width}Conditional width consistency of the pattern-moment coarse-graining. (a) Representative output trajectories at different widths approach the deterministic population-moment flow. (b) Initial-state error and (c) sup-time output error both exhibit a width-scaling compatible with $m^{-1/2}$.}
\end{figure*}

The preceding 15 moments determine outputs only on the specified training inputs; they need not determine the network's full input--output function. To require exact reconstruction for every real input $x$, restrict to a regular subfamily with $w_i>0$, pairwise distinct ReLU kink knots $t_i=-b_i/w_i$, and nonzero slope jumps $\gamma_i=a_iw_i/m$. The same network is then a piecewise-linear spline:
\begin{equation}
 f_{\thv}(x)=\sum_{i=1}^{m}\gamma_i(x-t_i)_+
 \label{eq:spline}
\end{equation}
For each neuron, varying $\gamma_i$ changes $f$ by the ramp $(x-t_i)_+$, while varying $t_i$ changes it in the independent (piecewise constant) direction $-\gamma_i\mathbf1_{\{x>t_i\}}$ away from the knot. These $2m$ perturbations are linearly independent for distinct knots and nonzero jumps; equivalently, finitely many suitably placed output evaluations have a Jacobian of rank $2m$. Therefore the observation-rank bound forces any smooth, exact, whole-function macrostate to satisfy $\dim Z_m\ge2m$ and $\dim Z_m/(3m)\ge2/3$.

\emph{Conclusions.---}
We have established a framework for dynamical coarse-graining of neural-network training based on the learning operator and three theorems concerning operator descent, exact closure, and dynamical coarse-graining. For a finite-width nonlinear ReLU network, we derived an exact macroscopic system consisting of five moments per fixed activation pattern. Numerical results show that 15 moments for three kinds of activation patterns reproduce the training outputs and evolving learning operator of a 720-parameter network within numerical accuracy, with substantially higher accuracy than the frozen-NTK approximation. The results demonstrate the feasibility of macroscopic descriptions of neural-network training under specified conditions and provide a basis for investigating coarse-grained dynamics in more general networks.

\begin{acknowledgments}
The work is supported by the National Natural Science Foundation of China (Grant No. 12588301)
\end{acknowledgments}

\appendix
\section{Appendix: Proofs of the three theorems}

\subsection{Proof of Theorem 1}
\begin{proof}
If $\mathbf K_{\thv}=\overline{\mathbf K}\circ h$,
then $\mathbf K_{\thv_1}=\mathbf K_{\thv_2}$ whenever
$h(\thv_1)=h(\thv_2)$, proving (i)$\Rightarrow$(ii).
Conversely, assume (ii). Since $h$ is surjective, each
$\uv\in\mathcal U$ has a preimage $\thv_{\uv}\in\Theta$.
Set
\begin{equation}
 \overline{\mathbf K}(\uv):=\mathbf K_{\thv_{\uv}}.
\end{equation}
Fiberwise constancy guarantees that this is independent of the
representative. Fix $\uv_0\in\mathcal U$.
There is a local smooth section
$s:W\to\Theta$ through a preimage of $\uv_0$.
For $\uv\in W$,
\begin{equation}
 \overline{\mathbf K}(\uv)=\mathbf K_{s(\uv)},
\end{equation}
so the descended map is smooth on $W$. Such neighborhoods cover
$\mathcal U$; the local expressions coincide where their domains
overlap because the definition was independent of representative.
Thus $\overline{\mathbf K}$ is globally smooth.
Surjectivity also implies uniqueness, since any two descended maps
must agree when evaluated at $h(\thv)$ for some $\thv$.

If (i) holds and $\bm v\in\ker Dh(\thv)$, the chain rule gives
\begin{equation}
 D\mathbf K_{\thv}[\bm v]
 =D\overline{\mathbf K}(h(\thv))[Dh(\thv)\bm v]=0.
\end{equation}
For the converse assume the vertical derivative vanishes at every
$\thv$. Since $h$ is a submersion, each fiber is a smooth
embedded submanifold of $\Theta$ with tangent space
$T_{\thv}h^{-1}(\uv)=\ker Dh(\thv)$.
If the fiber is connected, it is piecewise $C^1$ path connected.
Let $\gamma:[0,1]\to h^{-1}(\uv)$ be such a path joining
$\thv_1$ to $\thv_2$. At each point of differentiability,
$Dh(\gamma(t))\dot\gamma(t)=0$, hence
\begin{equation}
 \frac{\dd}{\dd t}\mathbf K_{\gamma(t)}
 =D\mathbf K_{\gamma(t)}[\dot\gamma(t)]=0.
\end{equation}
Integration on each smooth path segment gives
$\mathbf K_{\thv_1}=\mathbf K_{\thv_2}$.
By the already proved equivalence, (i) follows.
The connectedness hypothesis cannot be omitted from this
infinitesimal implication: vanishing tangential derivative controls
only the individual connected components of a fiber.

Let $\mathscr F$ be a family of smooth losses whose gradients span
$\RR^N$ at each $\uv$. If (i) holds, the projected velocity for
$\ell\in\mathscr F$ is
$-\overline{\mathbf K}(\uv)\nabla_{\uv}\ell(\uv)$,
which is the same at every microscopic point over $\uv$.
Conversely, suppose all these velocities agree for
$\thv_1,\thv_2\in h^{-1}(\uv)$. Then
\begin{equation}
 [\mathbf K_{\thv_1}-\mathbf K_{\thv_2}]
 \nabla_{\uv}\ell(\uv)=\mathbf0
 \quad(\ell\in\mathscr F).
\end{equation}
A linear map that vanishes on a spanning subset of $\RR^N$
is zero; thus $\mathbf K_{\thv_1}=\mathbf K_{\thv_2}$.
The equivalence already proved now yields (i).
\end{proof}

\subsection{Proof of Theorem 2}
\begin{proof}
Let $\mathbf W:\Theta\to\RR^r$ be the smooth map
$\mathbf W(\thv):=D\Pi(\thv)\Fv(\thv)$.
If $\mathbf W=\Gv\circ\Pi$, then for any
$\thv_1,\thv_2$ in the same fiber of $\Pi$,
$\mathbf W(\thv_1)=\mathbf W(\thv_2)$.
Conversely, assume $\mathbf W$ is constant on each fiber.
The formula
\begin{equation}
 \Gv(\zv):=\mathbf W(\thv),\qquad
 \text{for any }\thv\in\Pi^{-1}(\zv),
\end{equation}
is therefore well-defined. The surjective submersion $\Pi$ yields smooth local sections
$s:W\to\Theta$, so on each such neighborhood,
$\Gv|_W=\mathbf W\circ s$ is smooth. The local expressions
agree on intersections, proving that $\Gv$ is globally smooth.
Uniqueness follows from surjectivity of $\Pi$.

For every $i=1,\ldots,r$,
\begin{equation}
 [D\Pi(\thv)\Fv(\thv)]_i
 =D\phi_i(\thv)\Fv(\thv)
 =\mathcal D_{\Fv}\phi_i(\thv).
\end{equation}
Thus the vector relation $\mathbf W=\Gv\circ\Pi$ is equivalent
to the $r$ scalar equations
$\mathcal D_{\Fv}\phi_i=G_i\circ\Pi$.
This proves (i)$\Leftrightarrow$(iii) after the equivalence
(i)$\Leftrightarrow$(ii).

Let $\thv(t)$ solve
$\dot\thv=\Fv(\thv)$ with initial value $\thv_0$ and remain in
$\Theta$ for $t$ in an interval $I$ containing $0$.
The curve $\zv(t):=\Pi(\thv(t))$ is differentiable and
\begin{equation}
 \dot\zv(t)=D\Pi(\thv(t))\dot\thv(t)
 =\Gv(\Pi(\thv(t)))=\Gv(\zv(t)).
\end{equation}
Its initial value is $\zv(0)=\Pi(\thv_0)$.
Since $\Gv$ is smooth, local uniqueness for ODEs identifies this
curve with the reduced solution from that initial state.
Writing $\Phi_t$ for the microscopic solution map and
$\Psi_t$ for the reduced solution map therefore gives
\begin{equation}
 \Pi(\Phi_t(\thv_0))=\Psi_t(\Pi(\thv_0))
\end{equation}
whenever both sides are defined. No global completeness or
invariance beyond the common existence interval is assumed.
\end{proof}

\subsection{Proof of Theorem 3}
\begin{proof}
For each width $m$, let $\zv_m(t)$ solve
$\dot\zv_m=\Gv_m(\zv_m)$, and let $\zv(t)$ solve
$\dot\zv=\Gv(\zv)$, with both paths in the stated common region on
$[0,T]$. Integrating the two equations and inserting
$\Gv_m(\zv(t))$ yields
\begin{align}
 \zv_m(t)-\zv(t)
 ={}&\zv_m(0)-\zv(0)\\
 &+\int_0^t[\Gv_m(\zv_m(s))-\Gv_m(\zv(s))]\dd s\\
 &+\int_0^t[\Gv_m(\zv(s))-\Gv(\zv(s))]\dd s.
\end{align}
If all $\Gv_m$ are $L$-Lipschitz on a region containing both
trajectories and
$\sup\|\Gv_m-\Gv\|\le\epsilon_m$, the norm satisfies
\begin{equation}
 \|\zv_m(t)-\zv(t)\|
 \le\delta_m+L\int_0^t\|\zv_m(s)-\zv(s)\|\dd s+t\epsilon_m.
\end{equation}
Gr\"onwall's inequality, first with
$t\epsilon_m\le T\epsilon_m$ and then uniformly in $t\in[0,T]$,
gives
\begin{equation}
 \sup_{0\le t\le T}\|\zv_m(t)-\zv(t)\|
 \le e^{LT}(\delta_m+T\epsilon_m).
\end{equation}
The assumptions $\delta_m\to0$, $\epsilon_m\to0$, and
$P_m\to\infty$ with fixed $r$ imply the claimed trajectory
convergence and $r/P_m\to0$. No assertion of global gate preservation
or of convergence of the microscopic parameter trajectories is made.
\end{proof}

\begin{proof}[Proof of Theorem 3(b)]
Let $\mathcal O_m:\Theta_m\to\RR^{d_m}$ factor smoothly through
$\widehat\Pi_m:\Theta_m\to\widehat Z_m\subset\RR^{r_m}$:
$\mathcal O_m=\mathcal R_m\circ\widehat\Pi_m$.
At each $\thv\in\Theta_m$ the chain rule yields
\begin{equation}
 D\mathcal O_m(\thv)
 =D\mathcal R_m(\widehat\Pi_m(\thv))
 D\widehat\Pi_m(\thv).
\end{equation}
Since $D\widehat\Pi_m(\thv)$ is an $r_m\times P_m$ matrix,
\begin{equation}
 \operatorname{rank}D\mathcal O_m(\thv)
 \le\operatorname{rank}D\widehat\Pi_m(\thv)
 \le r_m.
\end{equation}
This is the desired pointwise bound. In particular, when a regular
family has rank at least $cP_m$ with $c>0$, exact smooth state
factorization forces $r_m/P_m\ge c$, excluding asymptotically
vanishing relative dimension on that family.
\end{proof}

\end{document}